\documentclass[11pt]{article}

\usepackage[margin=1in]{geometry}
\usepackage{amsmath,amssymb,amsfonts}
\usepackage{graphicx}
\usepackage{booktabs}
\usepackage{hyperref}
\usepackage{cleveref}

\usepackage[utf8]{inputenc} 
\usepackage[T1]{fontenc}    
\usepackage{url}            
\usepackage{caption}       
\usepackage{amsthm}         
\usepackage[ruled,vlined]{algorithm2e}
\usepackage{nicefrac}       
\usepackage{microtype}      
\usepackage{xcolor}         

\newtheorem{theorem}{Theorem}
\newtheorem{corollary}[theorem]{Corollary}
\newtheorem{lemma}[theorem]{Lemma}
\newtheorem{definition}{Definition}

\newcommand\ceil[1]{\left\lceil#1\right\rceil}
\newcommand{\methodname}{Quadrature-TreeSHAP}
\newcommand{\methodshort}{Q-TreeSHAP}

\title{\methodname: Depth-Independent TreeSHAP and Shapley Interactions}

\author{
  Ron Wettenstein\thanks{Equal contribution.} \\
  Reichman University, Herzliya, Israel \\
  \texttt{ron.wettenstein@post.runi.ac.il}
  \and
  Rory Mitchell\footnotemark[1] \\
  Nvidia Corporation \\
  \texttt{ramitchellnz@gmail.com}
  \and
  Peng Yu\footnotemark[1] \\
  Shopify \\
  \texttt{peng.yu@shopify.com}
}
\date{}

\begin{document}

\maketitle

\begin{abstract}
Shapley values are a standard tool for explaining predictions of tree ensembles, with Path-Dependent SHAP being the most widely used variant. Despite substantial progress, existing methods still exhibit trade-offs between depth-dependent runtime, numerical stability, and support for higher-order interactions. To address these challenges, we introduce \methodname{}, a quadrature-based reformulation of Path-Dependent TreeSHAP that is numerically stable, naturally extends to any-order Shapley interaction values and is practically insensitive to tree depth. 
Our implementation supports both CPU and GPU and is integrated into XGBoost.

Our method is based on a weighted-Banzhaf interaction polynomial, which expresses Banzhaf interaction values as expectations under a feature participation probability $p$. Shapley values and any-order interaction values are then recovered by integrating these polynomials over $p$ from 0 to 1. We evaluate these integrals using Gauss--Legendre quadrature, and show that, in practice, only $8$ fixed quadrature points are sufficient to reach machine precision. In fact, \methodname{} with $8$ fixed points achieves greater numerical stability than TreeSHAP. This fixed-point formulation removes depth dependence from the inner computation and enables efficient SIMD execution. 

We confirm these advantages empirically. On 12 XGBoost benchmarks, \methodname{} computes Shapley values $1.06\times$--$10.59\times$ faster than TreeSHAP on CPU and $1.84\times$--$6.95\times$ faster than GPUTreeSHAP on GPU. Shapley pairwise interactions are $3.80\times$--$58.11\times$ faster on CPU, with higher-order interactions achieving speedups of up to $1200\times$ compared to TreeSHAP-IQ.
\end{abstract}

\section{Introduction}
\label{sec:intro}

Decision tree ensembles such as XGBoost~\cite{chen2016xgboost}, Random Forests~\cite{breiman2001random}, and CatBoost~\cite{catboost} are widely used for classification and regression. As their use has grown, so has the need to explain their predictions. SHAP~\cite{lundberg2017unified} addresses this by framing explanation as a cooperative game and assigning each feature a Shapley value~\cite{shapley1953value}. While computing Shapley values on certain games and machine learning models is \#P-hard~\cite{shapley_calculation_is_npc, complexity_of_SHAP, tractability_shap_explanations}, the additivity structure of decision trees allows their efficient computation~\cite{lundberg2020local}. Beyond individual attributions, SHAP also extends to Shapley interaction values, which capture pairwise feature interactions~\cite{PB_shap_and_banzhaf, shapley_taylor_interaction, fumagalli2023shapiq}.

SHAP is commonly used with decision-tree models to understand model behavior and, in some cases, properties of the underlying data. Applications span finance~\cite{credit_risk_shap_xgboost,shap_load_default,shap_in_credit_fraud}, medicine~\cite{shap_heart_attack,shap_thyroid}, ecology and environmental science~\cite{shap_for_rivier_health,shap_for_water_quality,shap_for_plants}, disaster prediction and prevention~\cite{shap_for_landslide,shap_for_fire_prevention}, law enforcement~\cite{widlife_trading_airports,shap_law_enforcement}, and other domains. 

Tree ensembles are also used as surrogate-based explanation tools for more complex predictors and pipelines, since they can be explained efficiently with TreeSHAP~\cite{lundberg2020local}. A tree model is trained to mimic a black-box system, and TreeSHAP is then applied to approximate the original model's feature contributions. This paradigm has been explored in recent work~\cite{black_box_surrogate,ts_shap,surrogate_shap,hyper_shap}.


SHAP on decision tree ensembles has two commonly used variants. Background SHAP~\cite{laberge2022understandinginterventionaltreeshap} integrates over an explicit background dataset and can be computationally expensive, while Path-Dependent SHAP~\cite{lundberg2020local} efficiently estimates the effect of missing features using the training statistics stored along the tree. Although Path-Dependent SHAP can be inconsistent, providing different explanations for equivalent decision trees (as shown in~\cite{FastPD, symmetric_trees}), its efficient computation makes it the most widely used approach in practice.

TreeSHAP~\cite{lundberg2020local} is the first proposed algorithm for Path-Dependent SHAP and GPUTreeSHAP~\cite{GPUTreeShap} provides a GPU implementation of the approach. Both run in $O(MTLD^{2})$ time, where $M$ is the number of samples, $T$ is the number of trees, $L$ the number of leaves, and $D$ the maximum tree depth.
FastTreeSHAP v2~\cite{yang2022fast} and Woodelf~\cite{woodelf} reduce the complexity to $O(MTLD)$ time, but introduced a preprocessing step that is independent of $M$ but exponential in $D$, making them impractical for high depths trees. FastTreeSHAP v1~\cite{yang2022fast} reduces the complexity of TreeSHAP by constant factors.

Linear TreeSHAP~\cite{yu2022linear} recast the computation as polynomial arithmetic, achieving $O(MTND)$ time ($N$ = nodes), but suffers from numerical instability (see \cite{li2026treegrad}, Section 5). This polynomial view is further examined in~\cite{banzhaf_for_decision_trees}, which compares Shapley and Banzhaf attributions and studies their computational and theoretical trade-offs for tree models. TreeGrad-Shap~\cite{li2026treegrad} derives a gradient-based formulation via the multilinear extension, achieving $O(MTN\ceil{d/2})$, where $d$ is the maximal number of unique features along a root-to-leaf path ($d \le D$), and having no numerical instability. 

TreeGrad handles only first-order feature attributions. TreeSHAP-IQ~\cite{muschalik2024treeshapiq, NEURIPS2024_shapiq} extends Linear TreeSHAP polynomial arithmetic to any order interactions. Its complexity, when computing order-$s$ interaction on a dataset with $F$ features is $O(MTLD \binom{F-1}{s-1})$. In this paper, we provide an any-order Shapley interaction algorithm based on a gradient formulation that significantly improves the complexity of TreeSHAP-IQ.

\paragraph{Our contributions} are as follows:
\begin{enumerate}
\item We present an efficient method for computing any-order Shapley interaction values, improving the state of the art complexity of TreeSHAP-IQ from $O(MTLD \binom{F-1}{s-1})$ to $O(MTN d^s)$. This yields substantial gains in the highly common case when the number of features exceeds the tree depth ($F > D \ge d$). 

\item Our approach computes Shapley values as integrals of a weighted Banzhaf polynomial, evaluated efficiently using Gauss–Legendre quadrature. We provide both theoretical and empirical analysis of the number of quadrature points required. We show that in practice, even in very high depths, 8 Gauss--Legendre points are enough to reach the machine precision. Using a constant number of quadrature points removes the dependence on $d$ in practice, reducing the effective complexity from $O(MTN\ceil{d/2})$ to $O(MTN)$. More generally, for order-$s$ interactions, it reduces from $O(MTN d^s)$ to $O(MTN d^{s-1})$.

\item We show how to exploit SIMD efficiently in this fixed-point formulation, and implement the resulting method in C++. Additionally, we give the first GPU implementation of this quadrature-based idea in XGBoost. Our implementation (CPU and GPU) achieves $1.06\times$--$10.59\times$ speedups over TreeSHAP and $1.84\times$--$6.95\times$ speedups over GPUTreeSHAP.

\item The method is open-sourced, replacing the TreeSHAP backend in the XGBoost library~\footnote{This pull request merged our C++ implementation of \methodname{} into XGBoost: \url{https://github.com/dmlc/xgboost/pull/12179}. The GPU integration is still ongoing.}~\footnote{Integrating our code into XGBoost also speeds up the widely used \texttt{shap} package: for Path-Dependent SHAP on XGBoost models, \texttt{TreeExplainer} calls the XGBoost implementation behind the scenes and returns its output. For example, \texttt{TreeExplainer(xgb\_model).shap(df)} with a recent XGBoost version will use our approach.}. The change will appear in the next XGBoost release (following 3.2.0). To use our implementation, call \texttt{xgb\_model.predict(pred\_contribs=True)} for Shapley values and \texttt{xgb\_model.predict( pred\_interactions=True)} for Shapley interaction values (same API as before).


\end{enumerate}

Together, these contributions yield a numerically stable, hardware-efficient, GPU-friendly, and interaction-aware reformulation of Path-Dependent SHAP, improving both theory and practice.

\section{Polynomial Construction for TreeSHAP}

We first recall the polynomial view of TreeSHAP introduced by Linear TreeSHAP. 

Consider a decision tree $T$ with internal nodes and leaves.
For a sample $x$, the root-to-leaf path to leaf $v$, marked as $path(v)$, passes through edges $e_1, \dots, e_k$.
Each edge $e$ has a \emph{split feature} $f(e)$ and an \emph{edge weight} $w_e = n_{\mathrm{child}} / n_{\mathrm{parent}}$, the proportion of training samples flowing through that edge.

Define the \emph{empty prediction} for leaf $v$ with \emph{leaf value} $val(v)$ as:
\begin{equation}
\label{eq:R_empty}
R_{\emptyset}^{v} = \mathrm{val}(v) \cdot \prod_{e \in \mathrm{path}(v)} w_e\,
\end{equation}
For each feature $i$ on the path, the \emph{marginal multiplier} is:
\begin{equation}
\label{eq:q_def}
q_i^{v} =
\begin{cases}
\prod_{e:\, f(e)=i} 1/w_e & \text{if } x \text{ satisfies all splits on } i, \\
0 & \text{otherwise}.
\end{cases}
\end{equation}
A feature not on the path has $q_i = 1$ and is called a \emph{null player}. Including it in any coalition does not change the prediction. Using the above notation, the Path-Dependent game $f_v:2^{[F]} \to \mathbb{R}$ associated with a leaf $v$, defined over all feature subsets $S \subseteq [F]$, can be written as~\cite{yu2022linear}:

\begin{equation} \label{eq:path_dependent_game}
f_v(S) = R_{S}^{v} = R_{\emptyset}^{v} \cdot \prod_{i \in S} q_i^v
\end{equation}

The combinatorial information along a decision path can be encoded by the following summary polynomial (where $M(v)$ is the set of distinct features on the root-to-leaf path of $v$):
\begin{equation}
G_v(y) := R_\emptyset^v \prod_{j \in M(v)} (q_j^v + y)
\end{equation}

Linear TreeSHAP~\cite{yu2022linear} computes the Shapley contribution $\phi_i^v$ of feature $i \in M(v)$ by removing its factor and applying a weighted linear function $\psi$ to the resulting polynomial. The $\psi$ function is a weighted inner product over polynomial coefficients with $a_s$ the coefficient of $y^s$:
\begin{equation} \label{eq:linear_tree_shap_poly}
\phi_i^v
=
(q_i^v - 1)\, \psi\!\left( \frac{G_v(y)}{q_i^v + y},\, |M(v)| \right),
\quad
\psi(A, m) = \frac{1}{m} \sum_{s=0}^{\deg(A)} \frac{a_s}{\binom{m-1}{s}}
\end{equation} 

In this formulation, the Shapley value reduces to a weighted sum of polynomial coefficients, with weights given by the binomial terms in $\psi$. Linear TreeSHAP represents the polynomial in interpolation form and recovers $\psi$ via a Vandermonde transform, while~\cite{banzhaf_for_decision_trees} operates directly on coefficient vectors. TreeSHAP-IQ extends Formula~\ref{eq:linear_tree_shap_poly} to higher-order interactions and derives a corresponding algorithm.



\section{The Weighted Banzhaf Interaction Values Polynomial}

Banzhaf values~\cite{banzhaf1965weighted}, provide a classical alternative to Shapley values. 
Prior work has studied Banzhaf-based explanations for decision trees~\cite{banzhaf_for_decision_trees,muschalik2024treeshapiq,li2026treegrad, woodelf}, as well as other settings~\cite{databanzhaf,abramovich2023banzhafvaluesfactsquery,liu2025kernelbanzhaffastrobust}.
Formally, for a game $f_v:2^{[F]} \to \mathbb{R}$, the Banzhaf value of player $i$, equals the difference in the expected game's profit, under a uniform distribution over all subsets, with and without $i$:

\begin{equation} \label{banzhaf_expectations_def}
\beta_i(f_v) = \mathbb{E}[f_v(S) \mid i \in S] - \mathbb{E}[f_v(S) \mid i \notin S]
\end{equation}

More generally, the uniform distribution can be replaced with a participation probability $p$, yielding the weighted Banzhaf value~\cite{Radzik1997}, which admits several axiomatic characterizations~\cite{Nowak2000}. Extending this to feature subsets leads to any-order weighted Banzhaf interactions~\cite{weighted_banzhaf}:

\begin{definition}[Any Order Weighted Banzhaf interaction values] \label{def:any_order_wbanzhaf}
For a set function $f_v:2^{[F]} \to \mathbb{R}$ and a subset $S \subseteq [F]$, define the discrete derivative
\[
\Delta_S f_v(T) := \sum_{L \subseteq S} (-1)^{|S|-|L|} f_v(T \cup L)
\]
The order-$|S|$ weighted Banzhaf interaction at inclusion probability $p \in [0,1]$ is then
\[
B_p^{(S)}(v) := \mathbb{E}_{T \sim p}\!\left[ \Delta_S f_v(T) \right]
\]
where each feature in $[F]\setminus S$ is included independently in $T$ with probability $p$. 
\end{definition}

The case $|S|=1$ recovers the standard weighted Banzhaf value, $|S|=2$ corresponds to pairwise interactions, $|S|=3$ captures interactions among three players, and so on. These higher-order objects are the natural Banzhaf-side analogue of Shapley interaction indices~\cite{PB_shap_and_banzhaf, interaction_values}. Our key observation is that, for decision trees, these quantities also admit a simple polynomial form.

\begin{theorem}[Leaf-level weighted Banzhaf polynomial]
\label{theorem:weighted_banzhaf_poly}
Let $v$ be a leaf, let $M(v)$ be the set of features appearing on the path to $v$. 
For $S \subseteq [F]$, the Path-Dependent weighted Banzhaf interaction value of $S$, where each feature in $[F]\setminus S$ is independently included in $T$ with probability $p$, is:
\begin{equation} \label{eq:weighted_banzhaf_iv_v2}
B_{p,v}^{(S)}
=
\begin{cases}
R_\emptyset^v
\left( \displaystyle\prod_{j \in S} (q_j^v - 1) \right)
\displaystyle\prod_{j \in M(v)\setminus S} \big( (1-p) + p q_j^v \big)
& S \subseteq M(v), \\[10pt]
0 & \text{otherwise}.
\end{cases}
\end{equation}
Equivalently, defining $\alpha_j^v := q_j^v - 1$ and 
$H_v(p) := \prod_{j \in M(v)} (1 + \alpha_j^v p)$, we obtain:

\begin{equation} \label{eq:weighted_banzhaf_iv}
B_{p,v}^{(S)}
=
\begin{cases}
R_\emptyset^v
\left( \displaystyle\prod_{j \in S} \alpha_j^v \right)
\displaystyle\frac{H_v(p)}{\prod_{j \in S} (1 + \alpha_j^v p)}
& S \subseteq M(v), \\[12pt]
0 & \text{otherwise}.
\end{cases}
\end{equation}
\end{theorem}

\paragraph{Proof sketch.} The proof proceeds in two steps. First, we derive a closed-form expression for the discrete derivative $\Delta_S f_v(T)$ by expanding the alternating sum and using the product structure of the leaf rule. We then take expectation over $T$ under participation probability $p$, and use linearity of expectations and independence of feature inclusion to obtain the stated polynomial form. The full proof is provided in Appendix~\ref{sec:weighted_banzhaf_proof}.

\section{From Weighted Banzhaf Interactions to Shapley Interactions}

The weighted Banzhaf family is especially useful because Shapley quantities can be recovered by integrating over the participation probability $p \in [0,1]$. For first-order values, this follows from classical results on multilinear extensions and probabilistic values~\cite{owen1972multilinear}.
More generally, this relation extends to higher-order interactions as part of the probabilistic interaction framework~\cite{weighted_banzhaf}: 
\begin{equation} \label{eq:ao_shap_is_intergral_of_ao_banzhaf}
\mathrm{SII}^{(S)}(f_v) = \int_0^1 B_p^{(S)}(f_v)\, dp
\end{equation}
Here, $\mathrm{SII}^{(S)}(f_v)$ denotes the order-$|S|$ Shapley interaction value. For completeness, we prove Eq.~\eqref{eq:ao_shap_is_intergral_of_ao_banzhaf} in Appendix~\ref{sec:appendix_sii_from_weighted_banzhaf}.

This observation is algorithmically important because the integral over $[0,1]$ can be computed very efficiently with Gauss--Legendre quadrature. If $B_p^{(S)}$ is a polynomial in $p$ of degree at most $D$, then the corresponding Shapley quantity can be computed exactly as a weighted sum of weighted Banzhaf interaction values evaluated at fixed quadrature nodes. Let $\{(t_m,w_m)\}_{m=1}^n$ be the Gauss--Legendre nodes and weights on $[0,1]$ (exact when $2n-1 \ge D$~\cite{davis1984numerical,golub1969calculation}). Applying this at the leaf level using the weighted Banzhaf polynomial (Eq.~\eqref{eq:weighted_banzhaf_iv}) yields:
\begin{equation} \label{eq:shap_as_weighted_banzhafs}
\mathrm{SII}_v^{(S)}
=
\int_0^1 B_{p,v}^{(S)}\, dp
\approx
\sum_{m=1}^{n} w_m \, B_{t_m, v}^{(S)}=R_\emptyset^v
\left( \prod_{j \in S} \alpha_j^v \right)
\sum_{m=1}^{n} w_m
\frac{H_v(t_m)}{\prod_{j \in S} (1 + \alpha_j^v t_m)}
\end{equation}

TreeGrad~\cite{li2026treegrad} was the first tree-explanation method to adopt this viewpoint. It operates on the gradients of the multilinear extension, which correspond to first-order weighted Banzhaf values. In fact, Lemma~1 in TreeGrad is the first-order special case of our weighted Banzhaf formula in Theorem~\ref{theorem:weighted_banzhaf_poly}.

Our weighted-Banzhaf polynomial extends directly to arbitrary interaction orders, enabling the same quadrature framework to compute both Shapley values and higher-order interactions. Moreover, this representation shows that the tree-induced integrand is a low-degree polynomial (see Eq.~\eqref{eq:weighted_banzhaf_iv_v2}), explaining why quadrature achieves high accuracy in both our method and TreeGrad.

\section{Our algorithm}

Our algorithm evaluates Eq.~\eqref{eq:shap_as_weighted_banzhafs} for all order-$s$ interactions and all leaves simultaneously by a single depth-first traversal and sums the resulting leaf contributions on the fly. The DFS propagates, for each node, the values of the path polynomial at a fixed set of Gauss--Legendre nodes and merges child subtrees by pointwise addition. In practice we use a fixed set of $8$ Gauss--Legendre nodes on $[0,1]$ and propagate all eight traces during the traversal. Sect.~\ref{sec:quadrature_point_requirements} and~\ref{sec:results} explain why 8 points suffice.

For repeated features, we use the telescoping idea originally proposed by Linear TreeSHAP, but apply it to arbitrary interaction order. Let $e$ be a split on feature $i$ with head $h(e)$, let $p_e$ be the accumulated multiplier at edge $e$, let $p_{e^\uparrow}$ be the value at the closest ancestor edge on the same feature, and let $p_j$ denote the current live multiplier of feature $j$ on the active path. Then for any set $S$ of size $s-1$ drawn from the active path and not containing $i$, the child subtree contributes:
\begin{equation}
\label{eq:our_telescoping_background}
SII^{(S \cup \{i\})} \mathrel{+}= \sum_{m=1}^{n} w_m \, H_{h(e)}(t_m)
\left(
\frac{p_e - 1}{1 + (p_e - 1)t_m}
-
\frac{p_{e^\uparrow} - 1}{1 + (p_{e^\uparrow} - 1)t_m}
\right)
\prod_{j \in S} \frac{p_j - 1}{1 + (p_j - 1)t_m}
\end{equation}
When $S = \emptyset$, Eq.~\eqref{eq:our_telescoping_background} reduces to the first-order update. The pseudocode for the general order-$s$ algorithm is provided in Alg~\ref{alg:our_dfs}.

\begin{algorithm}[!t]
\caption{The \methodname{} algorithm for order-$s$ interactions}
\label{alg:our_dfs}
\SetKwFunction{DFS}{DFS}
\SetKwProg{Fn}{Function}{:}{}
\KwIn{Tree $T$, sample $x$, interaction order $s$, fixed Gauss--Legendre nodes $\{t_m, w_m\}_{m=1}^{n}$}
\KwOut{All order-$s$ Shapley interaction values $\Phi[S]$ for $S \subseteq [F]$ with $|S|=s$}
$\Phi[\cdot] \gets 0$; \quad $p[\cdot] \gets \texttt{UNSEEN}$; \quad $A \gets \emptyset$\;
\Fn{\DFS{$v$, $\mathbf{c}$, $w_{\mathrm{prod}}$}}{
  \If{$v$ is a leaf}{
    \Return $\mathbf{c} \cdot \mathrm{val}(v) \cdot w_{\mathrm{prod}}$\;
  }
  Let $f \gets \mathrm{feature}(v)$, children $l,r$\;
  \ForEach{child $u \in \{l,r\}$}{
    $w_e \gets n_u/n_v$; \quad $s_e \gets [x \text{ satisfies edge } v{\to}u]$\;
    \uIf{$p[f] = \texttt{UNSEEN}$}{
      $p_e \gets 1/w_e$ if $s_e$, else $0$; \quad $p_{\uparrow} \gets 1$\;
    }
    \uElseIf{$p[f] = 0$}{
      $p_e \gets 0$; \quad $p_{\uparrow} \gets 0$\;
    }
    \Else{
      $p_e \gets p[f]/w_e$ if $s_e$, else $0$; \quad $p_{\uparrow} \gets p[f]$\;
    }
    $\alpha_e \gets p_e - 1$\;
    $\mathbf{c}' \gets \mathbf{c} \odot (1 + \alpha_e \, \mathbf{t})$\;
    \If{$p[f] \neq \texttt{UNSEEN}$ \textbf{and} $|p[f]-1| > \epsilon$}{
      $\mathbf{c}' \gets \mathbf{c}' \oslash (1 + (p[f]-1)\,\mathbf{t})$\;
    }
    Save $p[f]$ and whether $f \in A$; set $p[f] \gets p_e$; add $f$ to $A$\;
    $\mathbf{H}_u \gets$ \DFS{$u$, $\mathbf{c}'$, $w_{\mathrm{prod}}\cdot w_e$}\;
    $\boldsymbol{\delta}_e \gets \frac{\alpha_e}{1+\alpha_e\mathbf{t}} - \frac{p_{\uparrow}-1}{1+(p_{\uparrow}-1)\mathbf{t}}$\;
    \ForEach{$(s-1)$-subset $P \subseteq A \setminus \{f\}$}{
      $\boldsymbol{\gamma}_P \gets \displaystyle\prod_{j \in P} \frac{p[j]-1}{1+(p[j]-1)\mathbf{t}}$\;
      $\Phi[P \cup \{f\}] \mathrel{+}= \sum_{m=1}^{n} w_m \, \mathbf{H}_u[m] \, \boldsymbol{\delta}_e[m] \, \boldsymbol{\gamma}_P[m]$\;
    }
    Restore $p[f]$; if $f$ was not previously active, remove $f$ from $A$\;
  }
  \Return $\mathbf{H}_l + \mathbf{H}_r$\;
}
\DFS{root, $\mathbf{1}_n$, $1.0$}\;
\end{algorithm}

In the first-order case, our algorithm is algebraically equivalent to TreeGrad-Shap~\cite{li2026treegrad} (specifically, TreeGrad's vectorized Algorithm 2). 
The difference lies not in the computation itself, but in the derivation and scope: TreeGrad derives the method via gradients of the multilinear extension and is limited to first-order values, whereas our approach is based on the weighted-Banzhaf polynomial and reveals that this is simply the first-order case of a more general interaction-values framework. 

\section{Quadrature Point Requirements}
\label{sec:quadrature_point_requirements}

We established that Shapley values and interactions can be computed with Gauss--Legendre quadrature. A natural next question is how many quadrature points are needed. This section addresses that question from both a theoretical and a practical perspective.

\paragraph{Theoretical analysis}
For first-order values, TreeGrad shows that Gauss-Legendre quadrature with $\lceil D/2 \rceil$ points suffices. Their Shapley algorithm uses the tighter bound $\lceil \min(D, F)/2 \rceil$, where $F$ is the number of features. This can be further tightened to $\lceil d/2 \rceil$, where $d$ is the maximum number of unique features along a root-to-leaf path. 

We extend this bound to arbitrary interaction orders, showing that the order-$s$ Shapley interaction integrand has lower degree than the first-order one. For example, if $d=10$, then first-order values require $5$ points, while third-order interactions require only $\lceil(10-3+1)/2\rceil = 4$ points.

\begin{theorem}[Decreasing degree for order-$s$ interactions]
\label{thm:decreasing_degree_main}
Let $S \subseteq [F]$ with $|S| = s$. The integrand for the order-$s$ Shapley interaction value $\mathrm{SII}(S)$ is a polynomial of degree at most $d - s$. Therefore, Gauss--Legendre quadrature is exact whenever the number of points satisfies:
\begin{equation}
\label{eq:n_min_main}
n \geq \left\lceil \frac{d - s + 1}{2} \right\rceil
\end{equation}
\end{theorem}

\noindent\emph{Proof sketch.} For a fixed interaction set, each leaf contribution is obtained by canceling the linear factors associated with the interacting features, so every requested interaction lowers the degree by one (see Eq.~\eqref{eq:weighted_banzhaf_iv}). Since a root-to-leaf path can involve at most $d$ distinct features, after cancellation each leaf term has degree at most $d-s$, and summing over leaves does not increase that bound.  Gauss--Legendre exactness then yields the result. See full proof in Appendix~\ref{sec:quadrature_points_proof}.

\paragraph{Practical analysis}
In practice, we do not need accuracy beyond the floating-point noise floor: the smallest perturbation that can be meaningfully represented and propagated by the target arithmetic. For IEEE 754 single precision, which is the output format returned here and in TreeSHAP, the unit roundoff is $u=2^{-24}\approx 5.96\times 10^{-8}$, so improvements far below roughly $10^{-7}$ are no longer reliably observable in the final \texttt{float32} result~\cite{goldberg1991floating,higham2002accuracy}. Once the quadrature error is below this scale, using more quadrature points brings no practical benefit.

Gauss--Legendre quadrature is also known to converge very rapidly for smooth integrands and to be exact for polynomials up to degree $2n-1$~\cite{davis1984numerical,golub1969calculation}. Our integrands inherit this favorable structure from the tree path polynomial. Empirically we find that $8$ quadrature points are enough. In other words, although Theorem~\ref{thm:decreasing_degree_main} gives a worst-case bound, the practical requirement is smaller, and experiments show that $8$ points consistently reach the numerical noise floor.

\section{SIMD and GPU Friendly approach}

SIMD (Single Instruction, Multiple Data)~\cite{SIMD} is a parallel computing paradigm in which a single instruction operates on multiple data points simultaneously. This is achieved by packing several values (e.g. float32) into a wide register. Prior work has leveraged SIMD to accelerate decision tree inference and related tasks~\cite{SIMD_trees, SIMD_trees_microsoft}.

The same idea extends to GPUs, where it is known as SIMT (Single Instruction, Multiple Threads), with computation organized into lanes grouped into warps and executed simultaneously under a single instruction~\cite{GPUTreeShap,cuda}. GPUTreeSHAP~\cite{GPUTreeShap} is designed with this constraint in mind. It decomposes the computation into root-to-leaf path subproblems and applies a bin-packing step to group similar-sized tasks into warps, keeping GPU lanes well utilized. While effective, this approach introduces preprocessing and scheduling overhead before the SHAP computation itself.

Our quadrature formulation is well suited to SIMD and SIMT. Because the computation reduces to evaluating the same fixed set of $8$ quadrature nodes for every path, the arithmetic is regular and fully utilizes vector lanes. In our implementation, we vectorize $4$ points in parallel on CPU (using SSE instructions\footnote{Longer 8 word vector instructions are not used in XGBoost due to portability.}) and $32$ points in parallel on GPU (the size of a warp). This fixed-width structure also simplifies the GPU algorithm, eliminating the binning strategy used by GPUTreeSHAP and yielding significant speedups, as shown in Sect.~\ref{sec:results}.

\section{Experimental Results}
\label{sec:results}

\begin{table*}[!t]
\caption{Benchmark XGBoost ensembles. Small and large models use breadth-first growth, while sparse models use depth-first growth with a 512-leaf cap. Here $D$ denotes realized maximum depth and $d$ denotes realized maximum unique-feature depth.}
\hspace{0.5px}
\centering
\scriptsize
\setlength{\tabcolsep}{4pt}
\resizebox{\textwidth}{!}{%
\begin{tabular}{llrllrrr}
\toprule
Dataset & Setting & Trees & Growth & Target & $D$ & $d$ & Mean leaves/tree \\
\midrule
Adult & small & 10 & breadth-first & max depth 6 & 6 & 6 & 40.1 \\
Adult & large & 1000 & breadth-first & max depth 16 & 16 & 14 & 613.3 \\
Adult & sparse & 100 & depth-first & 512 leaves & 28 & 13 & 512.0 \\
\midrule
CalHousing & small & 10 & breadth-first & max depth 6 & 6 & 5 & 64.0 \\
CalHousing & large & 1000 & breadth-first & max depth 16 & 16 & 8 & 3702.8 \\
CalHousing & sparse & 100 & depth-first & 512 leaves & 20 & 8 & 512.0 \\
\midrule
CovType & small & 80 & breadth-first & max depth 6 & 6 & 6 & 4.9 \\
CovType & large & 8000 & breadth-first & max depth 16 & 16 & 14 & 2.4 \\
CovType & sparse & 800 & depth-first & 512 leaves & 24 & 18 & 4.0 \\
\midrule
Fashion-MNIST & small & 100 & breadth-first & max depth 6 & 6 & 6 & 57.5 \\
Fashion-MNIST & large & 10000 & breadth-first & max depth 16 & 16 & 16 & 295.7 \\
Fashion-MNIST & sparse & 1000 & depth-first & 512 leaves & 52 & 48 & 412.4 \\
\bottomrule
\end{tabular}%
}
\label{tab:benchmark_models}
\end{table*}

We implemented \methodname{} in C++ and on GPU for first- and second-order interactions, and in Python for higher-order cases. We evaluate \methodname{} on 12 XGBoost benchmark ensembles spanning four datasets and three regimes per dataset (small, large, and sparse). The datasets are Adult~\cite{adult_dataset} (48,842 rows, 14 features), CalHousing~\cite{calhousing_dataset} (20,640 rows, 8 features), CovType~\cite{covtype_dataset} (581,012 rows, 54 features), and Fashion-MNIST~\cite{fashion_mnist_dataset} (70,000 rows, 784 features). Table~\ref{tab:benchmark_models} summarizes the benchmark models. 

The first-order benchmark explains 1000 rows and reports standard Shapley values, while the second-order benchmark explains 100 rows and reports pairwise Shapley interaction values. These two benchmarks compare the XGBoost implementations of TreeSHAP, GPUTreeSHAP, and \methodname{}. The higher-order Python experiments reported below instead use a single explained row and compare Python implementations of TreeSHAP-IQ and \methodname{}: Table~\ref{tab:higher_order_python} reports third-order interaction runtimes across benchmark models, and Figure~\ref{fig:higher_order_python_scaling} shows CovType-sparse scaling from orders $2$ to $6$.

Appendix~\ref{sec:appendix_treegrad_python} additionally reports a direct first-order Python comparison against TreeGrad-Shap, but we compare primarily against TreeSHAP in the main text because the XGBoost/\texttt{shap} implementation is, in practice, the fastest available. All runtimes below are wall-clock times; missing entries are marked as ``--'' and denote either runs that exceeded the 600-second timeout or, for GPUTreeSHAP, models with depth greater than 32. The benchmarks were run on a machine with an Intel Xeon E5-2698 v4 @ 2.20GHz CPU and an NVIDIA Tesla V100-SXM2-32GB GPU.

\begin{table*}[!t]
\caption{First-order runtimes on 1000 explained rows. Times are wall-clock seconds; speedups are relative to TreeSHAP or GPUTreeSHAP.}
\hspace{0.5px}
\centering
\scriptsize
\setlength{\tabcolsep}{4pt}
\resizebox{\textwidth}{!}{%
\begin{tabular}{llrrrrrr}
\toprule
& & \multicolumn{3}{c}{CPU} & \multicolumn{3}{c}{GPU} \\
\cmidrule(lr){3-5}\cmidrule(lr){6-8}
Dataset & Setting & TreeSHAP & \methodshort{} & Speedup & GPUTreeSHAP & \methodshort{} & Speedup \\
\midrule
Adult & small & 0.004 & 0.002 & 2.26 & 0.009 & 0.001 & 6.82 \\
Adult & large & 7.842 & 1.308 & 6.00 & 1.615 & 0.602 & 2.68 \\
Adult & sparse & 0.609 & 0.108 & 5.64 & 0.128 & 0.051 & 2.51 \\
\midrule
CalHousing & small & 0.003 & 0.003 & 1.06 & 0.008 & 0.001 & 6.95 \\
CalHousing & large & 29.034 & 7.681 & 3.78 & 8.390 & 2.705 & 3.10 \\
CalHousing & sparse & 0.341 & 0.103 & 3.33 & 0.085 & 0.035 & 2.47 \\
\midrule
CovType & small & 0.004 & 0.002 & 1.87 & 0.008 & 0.001 & 6.00 \\
CovType & large & 0.178 & 0.031 & 5.73 & 0.036 & 0.019 & 1.84 \\
CovType & sparse & 0.036 & 0.009 & 3.90 & 0.019 & 0.004 & 5.55 \\
\midrule
Fashion-MNIST & small & 0.046 & 0.021 & 2.16 & 0.023 & 0.012 & 2.01 \\
Fashion-MNIST & large & 51.800 & 5.639 & 9.19 & 14.140 & 2.431 & 5.82 \\
Fashion-MNIST & sparse & 8.299 & 0.783 & 10.59 & -- & 0.372 & -- \\
\bottomrule
\end{tabular}%
}
\label{tab:first_order_results}
\end{table*}

\begin{table*}[!t]
\caption{Second-order runtimes on 100 explained rows. Times are wall-clock seconds; speedups are relative to TreeSHAP or GPUTreeSHAP.}
\hspace{0.5px}
\centering
\scriptsize
\setlength{\tabcolsep}{4pt}
\resizebox{\textwidth}{!}{%
\begin{tabular}{llrrrrrr}
\toprule
& & \multicolumn{3}{c}{CPU} & \multicolumn{3}{c}{GPU} \\
\cmidrule(lr){3-5}\cmidrule(lr){6-8}
Dataset & Setting & TreeSHAP & \methodshort{} & Speedup & GPUTreeSHAP & \methodshort{} & Speedup \\
\midrule
Adult & small & 0.190 & 0.010 & 19.68 & 0.008 & 0.001 & 6.29 \\
Adult & large & 32.542 & 2.415 & 13.48 & 3.057 & 0.460 & 6.65 \\
Adult & sparse & 2.643 & 0.207 & 12.77 & 0.216 & 0.038 & 5.71 \\
\midrule
CalHousing & small & 0.018 & 0.004 & 4.67 & 0.007 & 0.001 & 6.51 \\
CalHousing & large & 59.608 & 15.707 & 3.80 & 9.611 & 2.609 & 3.68 \\
CalHousing & sparse & 0.903 & 0.211 & 4.28 & 0.096 & 0.033 & 2.93 \\
\midrule
CovType & small & 0.059 & 0.005 & 12.95 & 0.012 & 0.013 & 0.89 \\
CovType & large & 2.838 & 0.049 & 58.11 & 0.068 & 0.159 & 0.43 \\
CovType & sparse & 0.686 & 0.013 & 51.21 & 0.041 & 0.027 & 1.48 \\
\midrule
Fashion-MNIST & small & 15.549 & 2.052 & 7.58 & 1.764 & 3.975 & 0.44 \\
Fashion-MNIST & large & \multicolumn{1}{c}{--} & 12.923 & \multicolumn{1}{c}{--} & 35.208 & 7.161 & 4.92 \\
Fashion-MNIST & sparse & \multicolumn{1}{c}{--} & 3.355 & \multicolumn{1}{c}{--} & \multicolumn{1}{c}{--} & 4.514 & \multicolumn{1}{c}{--} \\
\bottomrule
\end{tabular}%
}
\label{tab:second_order_results}
\end{table*}

\paragraph{First-order TreeSHAP}
Table~\ref{tab:first_order_results} reports first-order runtimes. On CPU, \methodname{} is faster than TreeSHAP on all 12 benchmarks, with speedups ranging from $1.06\times$ to $10.59\times$ and a median speedup of $3.84\times$. The largest first-order CPU gains appear on the Fashion-MNIST large and sparse ensembles, while even the smallest gain on CalHousing-small remains slightly favorable. On GPU, \methodname{} outperforms GPUTreeSHAP on every benchmark with a reported baseline, with speedups between $1.84\times$ and $6.95\times$.

\paragraph{Second-order TreeSHAP}
Table~\ref{tab:second_order_results} reports pairwise interaction runtimes. The CPU advantage widens substantially for second-order explanations: among the 10 benchmarks with a reported TreeSHAP baseline, speedups range from $3.80\times$ to $58.11\times$, with a median speedup of $12.86\times$. The largest CPU gains occur on the CovType large and sparse models, with consistently strong gains across the Adult benchmarks, where the fixed-size quadrature computation avoids the rapid growth in interaction overhead. On GPU, \methodname{} is faster on 8 of the 11 benchmarks with reported measurements and reaches $6.65\times$ speedup, although several workloads still favor GPUTreeSHAP, notably CovType-small, CovType-large, and Fashion-MNIST-small.

\paragraph{Numerical stability}
Figure~\ref{fig:efficiency_error_vs_depth} examines the most numerically challenging benchmark in Table~\ref{tab:benchmark_models}, Fashion-MNIST-sparse, by gradually increasing the maximum depth of the learned trees and measuring the resulting violation of the efficiency property for first-order Shapley values. TreeSHAP becomes unstable beyond depth $32$, whereas \methodname{} remains remarkably stable; even $6$ points give acceptable accuracy, and we see no practical improvement beyond $8$, underscoring the effectiveness of Gauss--Legendre quadrature for these smooth integrals. Although these experiments focus on first-order attributions, higher-order interactions are numerically easier in our framework: Theorem~\ref{thm:decreasing_degree_main} shows they require no more quadrature points than first-order values, and often fewer.

The instability of TreeSHAP is structural. Per leaf, TreeSHAP works in the monomial basis, whose coefficients scale binomially, and recovers each Shapley value as a Bernstein-weighted alternating sum of those coefficients, so any rounding in the inflated monomial terms surfaces as catastrophic cancellation, and these errors accumulate across the leaves of a deep ensemble. \methodname{} sidesteps the basis entirely: it evaluates the summary polynomial at a handful of Gauss--Legendre nodes in $[0,1]$, where every factor is bounded and positive so the product is forward-stable and no cancellation occurs, which is why $6$--$8$ quadrature points already achieves high accuracy. 

TreeSHAP numerical instability is also discussed and analyzed in~\cite{banzhaf_for_decision_trees}. That work shows that, for tree models, Banzhaf values often provide similar explanations to Shapley values while being more numerically robust. This provides additional intuition for the numerical stability of our approach: by expressing Shapley values as an average of weighted Banzhaf values (see Eq.~\eqref{eq:shap_as_weighted_banzhafs}), our method retains the Shapley target while benefiting from the numerical stability of Banzhaf-style computations.

\begin{center}
\begin{minipage}[t]{0.49\textwidth}
\vspace{0pt}
\centering
\includegraphics[width=\linewidth]{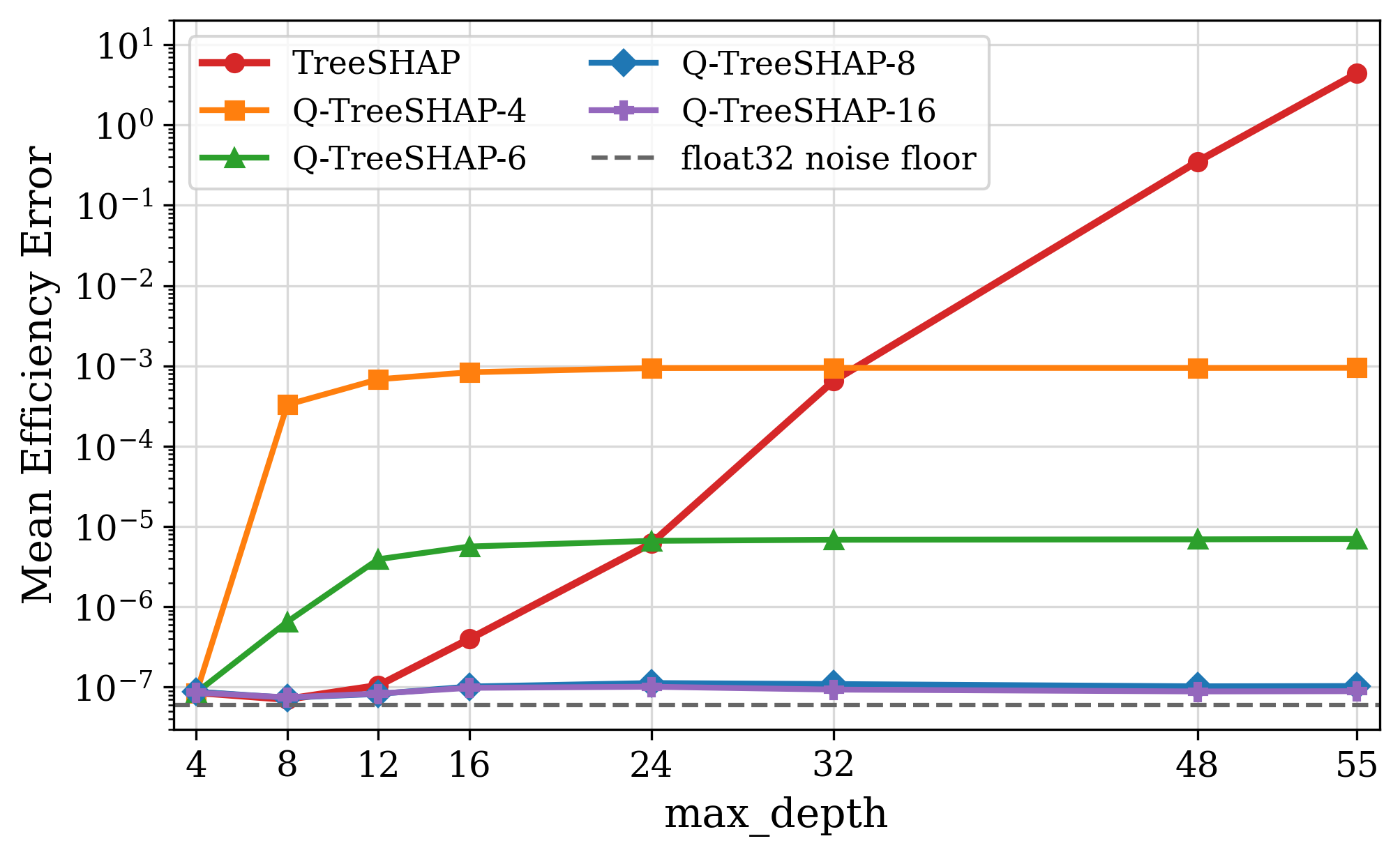}
\captionof{figure}{First-order numerical error on Fashion-MNIST-sparse vs. \texttt{max\_depth}.}
\label{fig:efficiency_error_vs_depth}
\end{minipage}\hfill
\begin{minipage}[t]{0.49\textwidth}
\vspace{0pt}
\centering
\includegraphics[width=\linewidth]{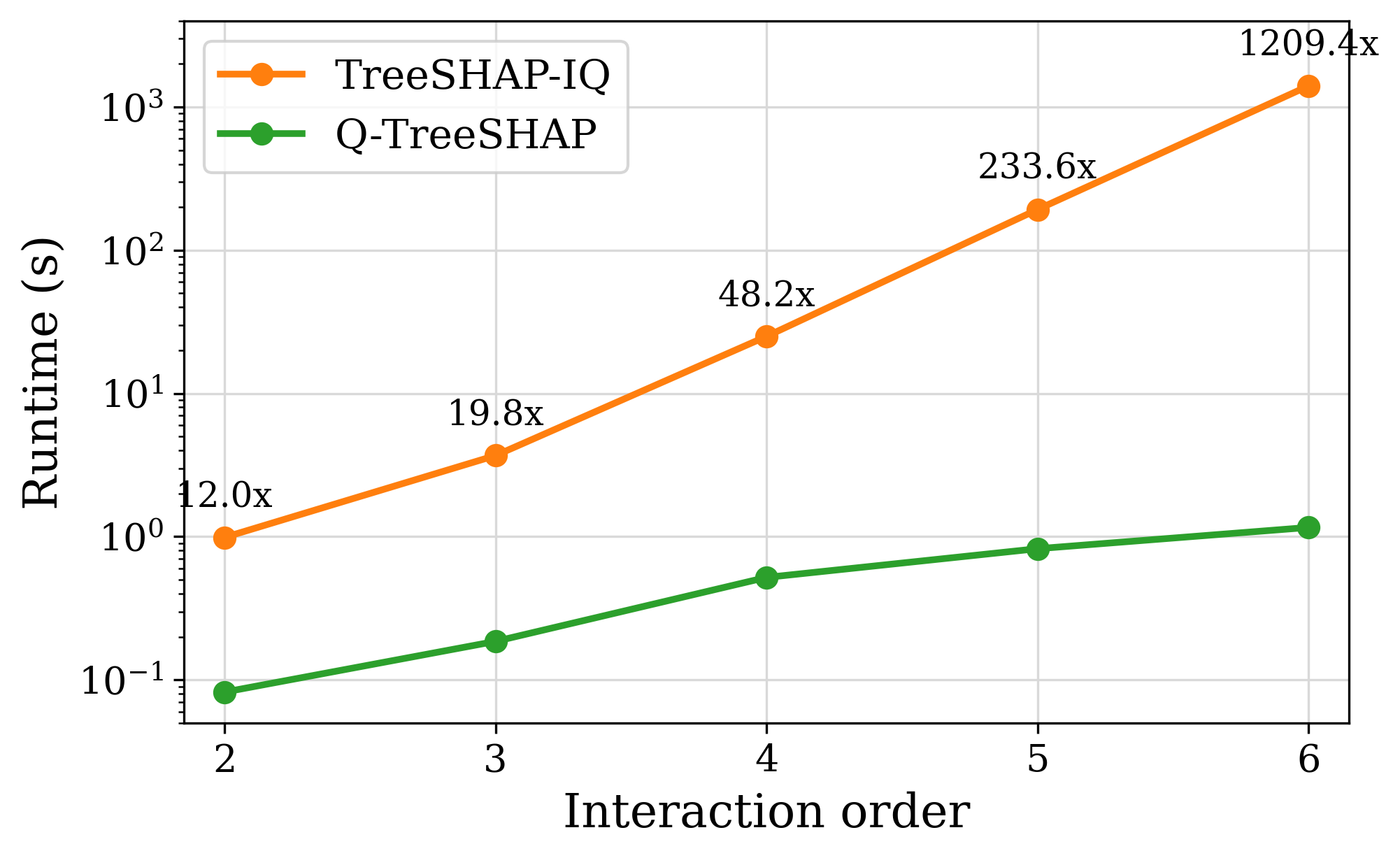}
\captionof{figure}{CovType-sparse runtime versus interaction order on a single explained row for TreeSHAP-IQ and \methodname{}.}
\label{fig:higher_order_python_scaling}
\end{minipage}
\end{center}

\begin{table}[!t]
\centering
\caption{Third-order interaction runtimes on a single explained row. Missing TreeSHAP-IQ entries indicate models with categorical XGBoost splits, which are unsupported by the available TreeSHAP-IQ adapter. CalHousing-large is excluded as both methods reached the 600s timeout.}
\label{tab:higher_order_python}
\footnotesize
\setlength{\tabcolsep}{4pt}
\begin{tabular}{lrrrrrr}
\toprule
Benchmark & $F$ & $D$ & $d$ & TreeSHAP-IQ & \methodshort{} & Speedup \\
\midrule
Adult-small & 14 & 6 & 6 & \multicolumn{1}{c}{--} & 0.010 & \multicolumn{1}{c}{--} \\
Adult-large & 14 & 16 & 14 & \multicolumn{1}{c}{--} & 37.616 & \multicolumn{1}{c}{--} \\
Adult-sparse & 14 & 28 & 13 & \multicolumn{1}{c}{--} & 2.873 & \multicolumn{1}{c}{--} \\
CalHousing-small & 8 & 6 & 5 & 0.253 & 0.014 & 18.73 \\
CalHousing-sparse & 8 & 20 & 8 & 30.027 & 1.362 & 22.04 \\
CovType-small & 54 & 6 & 6 & 0.182 & 0.009 & 20.90 \\
CovType-large & 54 & 16 & 14 & 25.489 & 0.962 & 26.49 \\
CovType-sparse & 54 & 24 & 18 & 3.685 & 0.186 & 19.84 \\
\bottomrule
\end{tabular}
\end{table}

\paragraph{Higher-order interactions}
To evaluate interaction orders beyond the pairwise XGBoost benchmark, we compare Python implementations of TreeSHAP-IQ and \methodname{} on the benchmark models (excluding Fashion-MNIST which has too many features). Table~\ref{tab:higher_order_python} reports third-order interaction runtimes on a single explained row. \methodname{} is consistently faster, with speedups ranging from $18.73\times$ to $26.49\times$; Adult models are shown without a TreeSHAP-IQ baseline because the available adapter does not support categorical XGBoost splits.

Figure~\ref{fig:higher_order_python_scaling} shows the CovType-sparse runtime trend on a single explained row as the interaction order increases from $2$ to $6$: the runtime gap widens rapidly, and the speedup grows from $12.02\times$ at order $2$ to $1209.37\times$ at order $6$, reflecting both the stronger order dependence of TreeSHAP-IQ and the practical benefit of our SIMD-friendly implementation.

\section{Conclusion}

We introduced \methodname{}, a quadrature-based reformulation of Path-Dependent TreeSHAP that removes practical depth dependence, remains numerically stable, and extends naturally to higher-order Shapley interaction values. Because the computation reduces to a small fixed set of evaluation points, it is also well suited to efficient CPU and GPU implementations.

In experiments on 12 XGBoost benchmarks, this structure translated into substantial gains: \methodname{} achieved CPU speedups of $1.06\times$ to $10.59\times$ for first-order attributions and $3.80\times$ to $58.11\times$ for second-order interactions, while on GPU it reached speedups of up to $6.95\times$ and $6.65\times$, respectively. In higher-order Python benchmarks on CovType-sparse, the advantage over TreeSHAP-IQ grew from $12.02\times$ at order $2$ to $1209.37\times$ at order $6$. These results show that our quadrature-based formulation yields a practical, stable, and hardware-efficient approach to fast tree-based attribution across interaction orders.

\bibliographystyle{plain}
\bibliography{reference}

\appendix

\section{Additional First-Order Python Comparison with TreeGrad}
\label{sec:appendix_treegrad_python}
To complement the XGBoost first-order benchmark in Sect.~\ref{sec:results}, Table~\ref{tab:first_order_python_treegrad} compares Python implementations of TreeGrad-Shap and \methodname{} on 10 explained rows. TreeGrad-Shap uses its model-dependent exact quadrature count $n=\lceil d/2 \rceil$, which ranges from 3 to 26 across these models, whereas \methodshort{} uses 8 fixed points. On the nine supported noncategorical benchmarks, \methodshort{} is $2.15\times$--$2.67\times$ faster, with a median speedup of $2.28\times$, while agreeing with TreeGrad-Shap up to a maximum absolute difference of $1.86 \times 10^{-11}$. Adult models are shown without a TreeGrad-Shap baseline because the available adapter does not support categorical XGBoost splits.

\begin{table*}[t]
\centering
\caption{First-order Python runtimes on 10 explained rows. TreeGrad-Shap uses its model-specific exact quadrature count $n=\lceil \min(D,F)/2 \rceil$, whereas \methodshort{} uses 8 fixed points. Times are total wall-clock seconds; speedups are relative to TreeGrad-Shap. Missing TreeGrad-Shap entries indicate models with categorical XGBoost splits, which are unsupported by the available TreeGrad adapter.}
\label{tab:first_order_python_treegrad}
\scriptsize
\setlength{\tabcolsep}{4pt}
\resizebox{\textwidth}{!}{%
\begin{tabular}{lrrrrrrr}
\toprule
Benchmark & $F$ & $D$ & $d$ & $n_{\mathrm{TG}}$ & TreeGrad-Shap & \methodshort{} & Speedup \\
\midrule
Adult-small & 14 & 6 & 6 & \multicolumn{1}{c}{--} & \multicolumn{1}{c}{--} & 0.050 & \multicolumn{1}{c}{--} \\
Adult-large & 14 & 16 & 14 & \multicolumn{1}{c}{--} & \multicolumn{1}{c}{--} & 69.498 & \multicolumn{1}{c}{--} \\
Adult-sparse & 14 & 28 & 13 & \multicolumn{1}{c}{--} & \multicolumn{1}{c}{--} & 5.663 & \multicolumn{1}{c}{--} \\
\midrule
CalHousing-small & 8 & 6 & 5 & 3 & 0.186 & 0.074 & 2.52 \\
CalHousing-large & 8 & 16 & 8 & 4 & 1079.476 & 410.543 & 2.63 \\
CalHousing-sparse & 8 & 20 & 8 & 4 & 14.772 & 5.682 & 2.60 \\
\midrule
CovType-small & 54 & 6 & 6 & 3 & 0.091 & 0.040 & 2.28 \\
CovType-large & 54 & 16 & 14 & 8 & 3.881 & 1.806 & 2.15 \\
CovType-sparse & 54 & 24 & 18 & 12 & 0.679 & 0.306 & 2.22 \\
\midrule
Fashion-MNIST-small & 784 & 6 & 6 & 3 & 1.970 & 0.737 & 2.67 \\
Fashion-MNIST-large & 784 & 16 & 16 & 8 & 738.263 & 325.558 & 2.27 \\
Fashion-MNIST-sparse & 784 & 52 & 48 & 26 & 104.485 & 45.954 & 2.27 \\
\bottomrule
\end{tabular}%
}
\end{table*}

\section{Weighted Banzhaf Polynomial Proof}
\label{sec:weighted_banzhaf_proof}
This section includes the proof of Theorem~\ref{theorem:weighted_banzhaf_poly}. We start by proving two required lemmas, and then continue to the main proof.

\subsection{Required Lemma Proofs}

\begin{lemma} \label{lemma_delta_f}
Let $f_v:2^{[F]} \to \mathbb{R}$ denote the Path-Dependent game associated with a leaf $v$, as defined in Eq.~\eqref{eq:path_dependent_game}. For $S \subseteq [F]$ and $T \subseteq [F]\setminus S$, define (same as in Def.~\ref{def:any_order_wbanzhaf}):
\[
\Delta_S f_v(T) := \sum_{L \subseteq S} (-1)^{|S|-|L|} f_v(T \cup L),
\]
Then:
\[
\Delta_S f_v(T)
=
\begin{cases}
R_\emptyset^v
\left( \displaystyle\prod_{j \in S} (q_j^v - 1) \right)
\left( \displaystyle\prod_{j \in T \cap (M(v)\setminus S)} q_j^v \right),
& S \subseteq M(v), \\[10pt]
0 & \text{otherwise}.
\end{cases}
\]
\end{lemma}

\begin{proof}
We start with the definition:
\[
\Delta_S f_v(T)
=
\sum_{L \subseteq S} (-1)^{|S|-|L|} f_v(T \cup L).
\]

Substituting $f_v$ and splitting the product using the identities below:
\begin{enumerate}
  \item By the distributive property: $(T \cup L) \cap M(v) = (T \cap M(v)) \cup (L \cap M(v))$.
  \item Since $T, S$ are disjoint: $T \cap M(v) = T \cap (M(v)\setminus S)$
  \item Since $T, S$ are disjoint and $L \subseteq S$, it follows that $T$ and $L$ are disjoint.
  \item Since $T, L$ are disjoint, $(T \cap (M(v)\setminus S))$ and $(L \cap M(v))$ are also disjoint.
\end{enumerate} 

\[
f_v(T \cup L)
=
R_\emptyset^v \prod_{j \in (T \cup L) \cap M(v)} q_j^v = R_\emptyset^v \left( \prod_{j \in T \cap (M(v)\setminus S)} q_j^v \right)
\left( \prod_{j \in L \cap M(v)} q_j^v \right).
\]

Thus, (the set $T \cap (M(v)\setminus S)$ is independent of $L$):
\[
\Delta_S f_v(T)
=
R_\emptyset^v
\left( \prod_{j \in T \cap (M(v)\setminus S)} q_j^v \right)
\sum_{L \subseteq S} (-1)^{|S|-|L|}
\prod_{j \in L \cap M(v)} q_j^v.
\]

By Lemma~\ref{lemma_subset_expansion_null_player} (see below):
\[
\sum_{L \subseteq S} (-1)^{|S|-|L|}
\prod_{j \in L \cap M(v)} q_j^v
=
\begin{cases}
\displaystyle \prod_{j \in S} (q_j^v - 1),
& S \subseteq M(v), \\[8pt]
0, & \text{otherwise}.
\end{cases}
\]
which completes the proof.
\end{proof}

\begin{lemma}[Subset Expansion and Null-Player Cancellation]
\label{lemma_subset_expansion_null_player}
Let $S \subseteq [F]$ and let $M \subseteq [F]$. For any coefficients $\{q_j\}_{j \in [F]}$, define
\[
\Phi(S, M) :=
\sum_{L \subseteq S} (-1)^{|S|-|L|}
\prod_{j \in L \cap M} q_j.
\]
Then:
\[
\Phi(S, M)
=
\begin{cases}
\displaystyle \prod_{j \in S} (q_j - 1),
& S \subseteq M, \\[8pt]
0, & \text{otherwise}.
\end{cases}
\]
\end{lemma}

\begin{proof}

\textbf{Case 1: $S \not\subseteq M$.}

Let $i \in S \setminus M$. Since $i \notin M$, the factor $q_i$ never appears in any product. Thus, for every $L \subseteq S \setminus \{i\}$,
\[
\prod_{j \in L \cap M} q_j
=
\prod_{j \in (L \cup \{i\}) \cap M} q_j.
\]

Consider the two terms indexed by $L$ and $L \cup \{i\}$. Their coefficients differ by a sign:
\[
(-1)^{|S|-|L|}
\quad \text{and} \quad
(-1)^{|S|-|L|-1}.
\]

Hence,
\[
(-1)^{|S|-|L|}
\prod_{j \in L \cap M} q_j
+
(-1)^{|S|-|L|-1}
\prod_{j \in (L \cup \{i\}) \cap M} q_j
= 0.
\]

Pairing all subsets in this way shows that every term cancels, and therefore
\[
\Phi(S, M) = 0.
\]

\medskip

\textbf{Case 2: $S \subseteq M$.}

In this case, since $S \subseteq M$, we have $L \cap M = L$ for all $L \subseteq S$. Hence
\[
\Phi(S, M)
=
\sum_{L \subseteq S} (-1)^{|S|-|L|}
\prod_{j \in L} q_j.
\]

We now show that this sum equals $\prod_{j \in S} (q_j - 1)$.

First, observe that
\[
\prod_{j \in S} (q_j - 1)
=
\prod_{j \in S} \bigl(-(1 - q_j)\bigr)
=
(-1)^{|S|} \prod_{j \in S} (1 - q_j).
\]

Next, we expand the product $\prod_{j \in S} (1 - q_j)$ using the standard subset expansion:
\[
\prod_{j \in S} (1 - q_j)
=
\sum_{L \subseteq S} \prod_{j \in L} (-q_j)
=
\sum_{L \subseteq S} (-1)^{|L|} \prod_{j \in L} q_j.
\]

Substituting this back, we obtain
\[
\prod_{j \in S} (q_j - 1)
=
(-1)^{|S|}
\sum_{L \subseteq S} (-1)^{|L|}
\prod_{j \in L} q_j
=
\sum_{L \subseteq S} (-1)^{|S|+|L|}
\prod_{j \in L} q_j.
\]

Considering the parity of both $|S|$ and $|L|$, we observe that $(-1)^{|S|+|L|} = (-1)^{|S|-|L|}$. Thus:
\[
\prod_{j \in S} (q_j - 1)
=
\sum_{L \subseteq S} (-1)^{|S|-|L|}
\prod_{j \in L} q_j.
\]

Comparing with the expression for $\Phi(S, M)$, we conclude that
\[
\Phi(S, M)
=
\prod_{j \in S} (q_j - 1),
\]
as claimed.

\end{proof}

\subsection{Theorem~\ref{theorem:weighted_banzhaf_poly} Proof}

Having established the required lemmas, we now provide the proof for Theorem~\ref{theorem:weighted_banzhaf_poly}. We take expectation over $T \sim p$ and assume $S \subseteq M(v)$ (otherwise the result is zero). Using the Lemma~\ref{lemma_delta_f}:
\[
\Delta_S f_v(T)
=
R_\emptyset^v
\left( \prod_{j \in S} (q_j^v - 1) \right)
\left( \prod_{j \in T \cap (M(v)\setminus S)} q_j^v \right).
\]

Taking expectation and extracting factors that do not depend on $p$ and $T$:
\[
B_{p,v}^{(S)}
=
R_\emptyset^v
\left( \prod_{j \in S} (q_j^v - 1) \right)
\mathbb{E}_{T \sim p}
\left[
\prod_{j \in T \cap (M(v)\setminus S)} q_j^v
\right].
\]

We rewrite the product using indicator variables, where $\mathbf{1}\{j \in T\}$ denotes whether feature $j$ is included in the random subset $T$. Then we use the independence of feature inclusion:
\[
\mathbb{E}_{T \sim p}
\left[
\prod_{j \in T \cap (M(v)\setminus S)} q_j^v
\right]
=
\mathbb{E}_{T \sim p}
\left[
\prod_{j \in M(v)\setminus S} (q_j^v)^{\mathbf{1}\{j \in T\}}
\right]
=
\prod_{j \in M(v)\setminus S}
\mathbb{E}_{T \sim p}
\left[
(q_j^v)^{\mathbf{1}\{j \in T\}}
\right],
\]

For each $j \in M(v)\setminus S$, we have
\[
\mathbb{E}_{T \sim p}
\left[
(q_j^v)^{\mathbf{1}\{j \in T\}}
\right]
=
(1-p)\cdot (q_j^v)^0 + p \cdot (q_j^v)^1
=
(1-p) + p q_j^v.
\]

Therefore,
\[
\mathbb{E}_{T \sim p}
\left[
\prod_{j \in T \cap (M(v)\setminus S)} q_j^v
\right]
=
\prod_{j \in M(v)\setminus S} \big( (1-p) + p q_j^v \big).
\]

Define:
\[
\alpha_j^v := q_j^v - 1.
\]

Then:
\[
(1-p) + p q_j^v = 1 + \alpha_j^v p.
\]

And:

\[
B_{p,v}^{(S)}
=
R_\emptyset^v
\left( \prod_{j \in S} (q_j^v - 1) \right)
\mathbb{E}_{T \sim p}
\left[
\prod_{j \in T \cap (M(v)\setminus S)} q_j^v
\right] = 
R_\emptyset^v
\left( \prod_{j \in S} \alpha_j^v \right)
\prod_{j \in M(v)\setminus S} \big( 1 + \alpha_j^v p \big).
\]

Note that by assigning $\alpha_j^v = q_j^v - 1$ in the polynomial above, we obtain Eq.~\eqref{eq:weighted_banzhaf_iv_v2}. 

Define the polynomial:
\[
H_v(p) := \prod_{j \in M(v)} (1 + \alpha_j^v p).
\]

Thus:
\[
\prod_{j \in M(v)\setminus S} (1 + \alpha_j^v p)
=
\frac{H_v(p)}{\prod_{j \in S} (1 + \alpha_j^v p)}.
\]

Recall that at the start of this section we assumed $S \subseteq M(v)$; otherwise, $\Delta_S f_v(T) = 0$. We obtain:
\[
\boxed{
B_{p,v}^{(S)}
=
\begin{cases}
R_\emptyset^v
\left( \displaystyle\prod_{j \in S} \alpha_j^v \right)
\displaystyle\frac{H_v(p)}{\prod_{j \in S} (1 + \alpha_j^v p)}
& S \subseteq M(v), \\[12pt]
0 & \text{otherwise}.
\end{cases}
}
\]
This proves Theorem~\ref{theorem:weighted_banzhaf_poly}.

\section{Proof of the quadrature-point theorem}
\label{sec:quadrature_points_proof}

In this section we prove Theorem~\ref{thm:decreasing_degree_main}.

\begin{proof}[Proof of Theorem~\ref{thm:decreasing_degree_main}]
Fix a set $S \subseteq [F]$ with $|S|=s$. By the weighted-Banzhaf polynomial formula, the leaf-level contribution can be written as
\[
\mathrm{SII}_v(S)
=
R_\emptyset^v
\left( \prod_{j \in S} \alpha_j^v \right)
\int_0^1 \frac{H_v(t)}{\prod_{j \in S}(1+\alpha_j^v t)}\,dt.
\]
Since
\[
H_v(t) = \prod_{j \in M(v)}(1+\alpha_j^v t),
\]
the denominator cancels the factors indexed by $S$, and therefore
\[
\frac{H_v(t)}{\prod_{j \in S}(1+\alpha_j^v t)}
=
H_v^{\setminus S}(t)
:=
\prod_{j \in M(v)\setminus S}(1+\alpha_j^v t).
\]
Hence
\[
\mathrm{SII}(S)
=
\int_0^1 \underbrace{\sum_{v:\, S\subseteq M(v)} R_\emptyset^v \left( \prod_{j \in S} \alpha_j^v \right) H_v^{\setminus S}(t)}_{I_S(t)}\,dt.
\]

For each leaf $v$, the polynomial $H_v^{\setminus S}(t)$ has degree $|M(v)|-s$. By the definition of $d$ (the maximum number of unique features along a root-to-leaf path), $|M(v)| \leq d$. Therefore,
\[
\deg H_v^{\setminus S}(t) \leq d-s.
\]
The total integrand $I_S(t)$ is a finite weighted sum of such polynomials, so it is itself a polynomial of degree at most $d-s$.

Finally, Gauss--Legendre quadrature with $n$ points is exact for all polynomials of degree at most $2n-1$~\cite{davis1984numerical,golub1969calculation}. Thus it is enough to require
\[
2n-1 \geq d-s,
\]
which is equivalent to
\[
n \geq \left\lceil \frac{d-s+1}{2} \right\rceil.
\]
This proves Theorem~\ref{thm:decreasing_degree_main}.
\end{proof}

\section{Recovering Shapley Interaction Values from Weighted Banzhaf Interactions Polynomial}
\label{sec:appendix_sii_from_weighted_banzhaf}

In this section we show that Shapley interaction values can be recovered by integrating
the weighted Banzhaf interaction values over the participation probability $p \in [0,1]$.

\begin{lemma}[Beta-function identity]
\label{lem:beta_identity}
For any integers $k \ge 0$ and $n \ge k$, we have
\[
\int_0^1 p^k (1-p)^{n-k}\,dp
=
\frac{k!(n-k)!}{(n+1)!}.
\]
\end{lemma}

\begin{proof}
This is the standard Beta-function identity:
\[
\int_0^1 p^{a-1}(1-p)^{b-1}\,dp
=
\frac{\Gamma(a)\Gamma(b)}{\Gamma(a+b)}.
\]
Setting $a=k+1$ and $b=n-k+1$, and using $\Gamma(m+1)=m!$ for integers $m$,
gives the result.
\end{proof}

\begin{theorem}[Shapley interaction values as an integral of weighted Banzhaf interactions]
\label{thm:sii_from_weighted_banzhaf}
Let $f_v:2^{[F]} \to \mathbb{R}$ be a set function and let $S \subseteq [F]$.
Then
\[
\mathrm{SII}(S) = \int_0^1 B_p^{(S)}(f_v)\, dp.
\]
\end{theorem}

\begin{proof}
By definition of the weighted Banzhaf interaction, where $T$ is sampled by including each feature in $[F]\setminus S$ independently with probability $p$ (i.e., $T \sim \mathrm{Bernoulli}(p)$):
\[
B_p^{(S)}(f_v)
=
\mathbb{E}_{T \sim p}\!\left[\Delta_S f_v(T)\right]
=
\sum_{T \subseteq [F]\setminus S}
p^{|T|}(1-p)^{F-|S|-|T|}
\Delta_S f_v(T),
\]
where
\[
\Delta_S f_v(T)
:=
\sum_{L \subseteq S} (-1)^{|S|-|L|} f_v(T \cup L).
\]

Integrating over $p$ and exchanging sum and integral,
\[
\int_0^1 B_p^{(S)}(v)\,dp
=
\sum_{T \subseteq [F]\setminus S}
\left(
\int_0^1 p^{|T|}(1-p)^{F-|S|-|T|}\,dp
\right)
\Delta_S f_v(T).
\]

Applying Lemma~\ref{lem:beta_identity} with $k=|T|$ and $n=F-|S|$, we obtain
\[
\int_0^1 B_p^{(S)}(v)\,dp
=
\sum_{T \subseteq [F]\setminus S}
\frac{|T|!(F-|S|-|T|)!}{(F-|S|+1)!}
\Delta_S f_v(T).
\]

This is exactly the definition of the Shapley interaction index~\cite{PB_shap_and_banzhaf}:
\[
\mathrm{SII}(S)
=
\sum_{T \subseteq [F]\setminus S}
\frac{|T|!(F-|S|-|T|)!}{(F-|S|+1)!}
\Delta_S f_v(T).
\]
Therefore,
\[
\mathrm{SII}(S) = \int_0^1 B_p^{(S)}(f_v)\,dp.
\]
\end{proof}

\begin{corollary}[First-order case]
\label{cor:first_order_weighted_banzhaf}
For $S=\{i\}$, the theorem reduces to
\[
\phi_i(f_v) = \int_0^1 B_p^{(\{i\})}(f_v)\,dp.
\]
\end{corollary}

\end{document}